\documentclass[letterpaper, 10pt]{IEEEtran} %%%%%%%%%%%%%%%%%% Fred added this
\usepackage[letterpaper, left=0.65in, right=0.65in, bottom=0.7in, top=0.7in]{geometry} %%%%%%%%%%%%%%%%%% Fred added this
\usepackage{graphicx}
\usepackage{float}

\usepackage{fixltx2e}
\usepackage{tabularx}
\usepackage{verbatim}
\usepackage{color}
\usepackage{hyperref}
\usepackage{xmpmulti}
\usepackage{transparent}
\usepackage{fancyhdr}
\usepackage{cite}

\usepackage{amssymb,amsmath,graphicx,float}
\usepackage{algorithm, algpseudocode}
\usepackage{booktabs}
\usepackage{caption}
\usepackage{subcaption}
\usepackage{amsthm}
\usepackage{hyperref}
\usepackage{wrapfig}
\usepackage{setspace}
\usepackage{graphicx}
\usepackage{epsfig}
\usepackage[all]{xy}
\usepackage{verbatim}
\usepackage{float}
\usepackage{mathrsfs}  
\usepackage{bm}
\usepackage{multirow}
\usepackage{color,soul}
\usepackage{amssymb,amsmath,graphicx}
\usepackage{amsthm}
\usepackage[compact]{titlesec}

\usepackage{array}
\newtheorem{theorem}{Theorem}

\newtheorem{problem}{Problem}
                       
\newtheorem{definition}{Definition}

\newtheorem{remark}{Remark}

\newcommand{\norm}[1]{\left\lVert#1\right\rVert}
\newcommand{\tabincell}[2]{\begin{tabular}{@{}#1@{}}#2\end{tabular}}
\let\oldIEEEkeywords\IEEEkeywords
\def\IEEEkeywords{\oldIEEEkeywords\normalfont\bfseries\ignorespaces}
\newcommand{\argmin}{\mathrm{arg}\min}

\begin{document}
\pagenumbering{gobble}
	
	\title{Controller Synthesis for Multi-Agent Systems with Intermittent Communication and Metric Temporal Logic Specifications}
	
	\author{Zhe~Xu, Federico M. Zegers, Bo Wu, Alexander J. Phillips, Warren Dixon and Ufuk Topcu
		\thanks{Zhe~Xu is with the School for Engineering of Matter, Transport and Energy, Arizona State University, Tempe, AZ 85287, Federico M. Zegers and Warren Dixon are with the Department of Mechanical
        and Aerospace Engineering, University of Florida, Gainesville, Florida 32611, Bo Wu is with the Oden Institute
		for Computational Engineering and Sciences, University of Texas, Austin, Austin, TX 78712, Alexander J. Phillips is with the Department
		of Electrical and Computer Engineering, University of Texas,
			Austin, Austin, TX 78712, Ufuk Topcu is with the Department
			of Aerospace Engineering and Engineering Mechanics, and the Oden Institute
			for Computational Engineering and Sciences, University of Texas,
			Austin, Austin, TX 78712, e-mail: xzhe1@asu.edu, fredzeg@ufl.edu, bwu3@utexas.edu, ajp3777@utexas.edu, wdixon@ufl.edu, utopcu@utexas.edu. This research is supported in part by AFOSR award numbers FA9550-18-1-0109 and FA9550-19-1-0169, and NEEC award number N00174-18-1-0003. }     
	}

% 	\thanks{Portions of this paper previously appeared as a conference paper: Z. Xu, F. M. Zegers, B. Wu, W. Dixon and U. Topcu, Controller Synthesis for Multi-Agent Systems With Intermittent Communication: A Metric Temporal Logic Approach, \textit{Proc. Annual Allerton Conference on Communication, Control, and Computing}, 2019.}

	\maketitle
	
\begin{abstract} 
This paper investigates the controller synthesis problem for a multi-agent system (MAS) with
intermittent communication. We adopt a \textit{relay-explorer} scheme, where a mobile \textit{relay agent} with absolute position sensors switches among
a set of \textit{explorers} with relative position sensors to
provide intermittent state information. We model the MAS as a switched system where the explorers' dynamics can be either fully-actuated or under-actuated. The objective of the explorers is to reach \textit{approximate consensus} to a predetermined goal region. To guarantee the stability of the switched system and the approximate consensus of the explorers, we derive \textit{maximum dwell-time conditions} to constrain the length of time each explorer goes without state feedback (from the relay agent). Furthermore, the relay agent needs to satisfy practical constraints such as charging its battery and staying in specific regions of interest. Both the maximum dwell-time conditions and these practical constraints can be expressed by \textit{metric temporal logic} (MTL) specifications. We iteratively compute the optimal control inputs for the relay agent to satisfy the MTL specifications, while guaranteeing stability and approximate consensus of the explorers. We implement the proposed method on a case study with the CoppeliaSim robot simulator.
\end{abstract}     
	
\section{Introduction}
\label{sec_intro}
    
% Coordination strategies for multi-agent systems (MAS)
% have been traditionally designed under the assumption that state feedback is continuously available. 
Traditionally, coordination strategies for multi-agent systems
(MAS) have been designed under the assumption that
state feedback is continuously available and each
agent can continuously communicate with its neighbors over
a network.
This assumption is often impractical, especially in mobile robot applications
where shadowing and fading in the wireless communication can cause unreliability, and each agent has limited energy resources \cite{goldsmith2005wireless}. 
	
Due to these constraints, there is a strong interest in developing
MAS coordination methods that rely on intermittent information over a communication network. In \cite{Wang.Lemmon2009,Meng.Chen2013,Cheng.Kan.ea2017,Li.Liao.ea2015,Heemels.Donkers2013,tabuada2007event}, the authors develop \textit{event-triggered} and \textit{self-triggered} controllers to only utilize sampled
	data from networked agents when triggered by conditions that ensure
	desired stability and performance properties. However, these results usually
	require a network represented by a strongly connected graph to enable agent coordination.
% Therefore, there is a need for distributed methods capable
% of coordinating these agents that are not equipped with
% absolute position sensors while utilizing intermittent information.
% Moreover, such methods should not require agents to perform additional
% maneuvers to ensure the connectivity of the network. 
In \cite{Zegers.Chen.ea2019}, the authors provided a framework where a
set of \textit{explorers} operating with inaccurate position sensors 
are able to reach consensus at a desired state while a \textit{relay agent} intermittently
provides each explorer with state information. By introducing a
relay agent, the explorers are able to perform their tasks without
the need to perform additional maneuvers to obtain state information.
	%   This requirement of a \textit{strongly connected network} induces 
%   constraints on the motion of the individual agents and additional maneuvers that may deviate
%   from their primary purpose. Event-triggered and self-triggered control methods
%   can also be used to coordinate the agents that communicate with a central base station or \textit{cloud}
%     intermittently as in \cite{Nowzari_Pappas2016}, where submarines
%   intermittently surface to obtain state information about themselves and
%   their neighbors from a cloud. However, such a coordination strategy
%   also requires additional maneuvers from the submarines that detract
%   from their primary purpose. 
	
%  Depending on the application and/or environment, some of the agents in
%  an MAS may not be equipped with absolute position sensors. In such
%  scenarios, the results in \cite{Wang.Lemmon2009,Meng.Chen2013,Cheng.Kan.ea2017,Li.Liao.ea2015,Heemels.Donkers2013,tabuada2007event}
% are invalid. 
Building on the work of \cite{Zegers.Chen.ea2019,Chen2019}, we adopt a \textit{relay-explorer} scheme, where the MAS is modeled as a switched system. As an illustrative example shown in Fig. \ref{fig_intro}, the three explorers need to reach \textit{approximate consensus} to the green goal region and one relay agent provides intermittent state information to each explorer. To guarantee stability of the switched system and approximate consensus of the explorers, we derive \textit{maximum dwell-time conditions} to constrain the intervals between consecutive time instants at which the relay agent should provide state information to the same explorer. 

The maximum dwell-time conditions can be encoded by \textit{metric temporal logic} (MTL) specifications \cite{Ouaknine2005}. Such specifications have also been used in robotic applications for time-related specifications \cite{zhe_ijcai2019}. Since the relay agent is typically more energy-consuming due to high-quality communication and mobility equipment, the relay agent is likely required to satisfy additional MTL specifications for practical constraints such as charging its battery and staying in specific regions. 
In the example shown in Fig. \ref{fig_intro}, the relay agent needs to satisfy an MTL specification ``\textit{reach the charging station $G_1$ or $G_2$ in every 6 time units and always stay in the purple region $D$}''. 
	
We design the explorers' controllers such that the guarantees on the stability of the switched system and approximate consensus of the explorers hold, provided that the maximum dwell-time conditions are satisfied. Then, we synthesize the relay agent's controller to satisfy the MTL specifications that encode the maximum dwell-time conditions and the additional practical constraints. There is a rich literature on controller synthesis subject to temporal logic specifications \cite{KHFP,Nok2012,BluSTL,sayan2016,zhe_advisory,zhe_control,liu2020distributed}. For linear or switched linear systems, the controller synthesis problem can be converted into a mixed-integer linear programming (MILP) problem~\cite{BluSTL,sayan2016}. Additionally, as the explorers are equipped with 
relative position sensors, we design an observer to estimate the explorers' states, which can jump due to the provision of intermittent state feedback via communication. Therefore, we solve the MILP problem iteratively to account for such abrupt changes. 

% Zhiyu2017ACC,Zhiyu2017CDC,,wu2015combined

This paper provides additional insights and generalizes our previous work in \cite{Xu2019Allerton}. (a) The proposed approach in \cite{Xu2019Allerton} only applies to fully-actuated or over-actuated dynamics for the explorers, while we extend the approach to under-actuated dynamics (e.g., unicycle dynamics) for the explorers in this paper. (b) We used both maximum and minimum dwell-time conditions
to achieve stability and approximate consensus for the explorers in \cite{Xu2019Allerton}, while in this paper we only rely on the maximum dwell-time conditions, i.e., the minimum dwell-time conditions are not necessary to enable the result. (c) This paper provides additional evidence of the approach through CoppeliaSim robot simulators with multiple MTL specifications in the case studies.
	
We implemented the proposed method on a simulation case study with three mobile robots as the explorers and one quadrotor as the relay agent. The results in two different scenarios show that the synthesized controller can lead to satisfaction of the MTL specifications, while ensuring the stability of the switched system and achieving the approximate consensus objective.

	\begin{figure}
		\centering
		\includegraphics[scale=0.3]{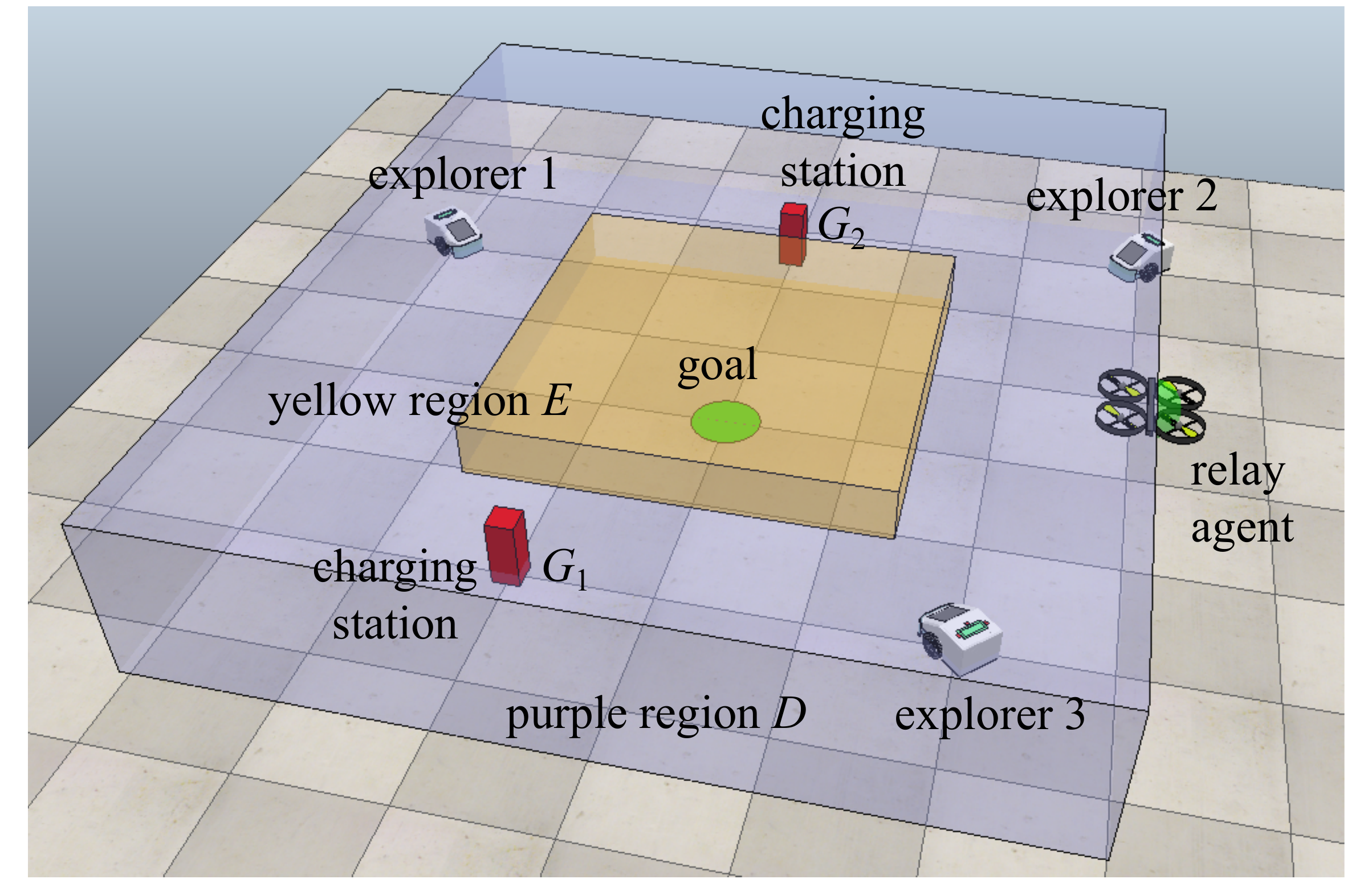}
		\caption{Illustrative example of an MAS with a relay agent (quadrotor) and three explorers (mobile robots).}  
		\vspace{-6mm}
		\label{fig_intro}
	\end{figure}

\section{Problem Formulation}

\subsection{Agent Dynamics}
\label{sec_dynamics}

Consider a multi-agent system (MAS) consisting of $Q$ explorers ($Q\in\mathbb{Z}_{>0}$) indexed by $i\in F\triangleq\left\{ 1,...,Q\right\}$
and a relay agent indexed by $0$. Let the time set be $\mathbb{T} \textcolor{black}{\triangleq} \mathbb{R}_{\ge0}$.
Let $y_{0},\text{ }y_{i}: \mathbb{T}\rightarrow\mathbb{R}^{z}$
denote the position of the relay agent and explorer $i$, respectively. Let $x_{0}: \mathbb{T}\rightarrow\mathbb{R}^{l}$ and $x_{i}: \mathbb{T}\rightarrow\mathbb{R}^{m}$
denote the state of the relay agent and explorer $i$, respectively.
The \textcolor{black}{known} linear time-invariant dynamics of the relay agent and explorer $i$
are 
\begin{align}
\begin{split}
\dot{x}_{0}\left(t\right) & =  A_0x_{0}\left(t\right)+B_0u_{0}\left(t\right), \\
 y_{0}\left(t\right) & =  C_0x_{0}\left(t\right), \\
\dot{x}_{i}\left(t\right) & =  Ax_{i}\left(t\right)+Bu_{i}\left(t\right)+d_{i}\left(t\right), \\
 y_{i}\left(t\right) & = Cx_{i}\left(t\right),
 \end{split}
\label{explorer Dynamics}
\end{align}
where $A_0\in\mathbb{R}^{l\times l}, A\in\mathbb{R}^{m\times m}$, $B_0\in\mathbb{R}^{l\times n}, B\in\mathbb{R}^{m\times n}$, $C_0\in\mathbb{R}^{z\times l}$, and $C\in\mathbb{R}^{z\times m}$. Here, $u_{0},\text{ }u_{i}: \mathbb{T}\rightarrow\mathbb{R}^{n}$
denote the control inputs of the relay agent and explorer $i$, respectively,
and $d_{i}: \mathbb{T}\rightarrow\mathbb{R}^{m}$ denotes an
exogenous disturbance that is \textcolor{black}{continuous and} bounded,
i.e., $\left\Vert d_{i}\left(t\right)\right\Vert \leq\overline{d}_{i}$
for all $t\geq0$ ($\overline{d}_{i}\in\mathbb{R}_{>0}$ is a
known constant)\footnote{$\left\Vert\cdot\right\Vert $ denotes the 2-norm.}. We assume that
the pair $(A,B)$ is stabilizable.

\subsection{Communication}
\label{sec_sensing}
Each explorer is equipped with a relative position sensor and hardware to enable communication with \textcolor{black}{other agents, e.g., the relay agent and a goal region}.
Since the explorers lack absolute position sensors, they are
not able to localize themselves within the global coordinate system.
Nevertheless, the explorers can use their relative position sensors
to enable self-localization relative to their initially known locations. However,
relative position sensors\textcolor{black}{,} like encoders and inertial measurement
units (IMUs)\textcolor{black}{,} can produce unreliable position information since wheels of mobile robots
may slip and IMUs may generate noisy data. Hence, the $d_{i}\left(t\right)$
term in (\ref{explorer Dynamics}) models the inaccurate position
measurements from the relative position sensor of explorer $i$ as
well as any external influences from the environment.
Navigation through the use of a relative position sensor results in
dead-reckoning, which becomes increasingly inaccurate with time
if not corrected. On the other hand, the relay agent is equipped with an absolute
position sensor and hardware to enable communication
with each explorer. Unlike a relative position sensor, an absolute
position sensor allows localization of the agent within the global coordinate
system.

\textcolor{black}{Let $x_{g}\in\mathbb{R}^{m}$ be a predetermined state defined by the user.} A goal region (see Fig. \ref{fig_intro}) centered at the position $Cx_{g}\in\mathbb{R}^{z}$ with radius $R_{f}\in\mathbb{R}_{>0}$ is capable of providing state information to each explorer $i\in F$
once $\left\Vert y_{i}\left(t\right)-Cx_{g}\right\Vert=\left\Vert Cx_{i}\left(t\right)-Cx_{g}\right\Vert \leq R_{f}$. 

Let $R\in\mathbb{R}_{>0}$
denote the communication radius of the relay agent and each explorer. Within this work, the relay agent has full knowledge of its own state $x_{0}\left(t\right)$ for all $t\geq0$ and the initial state $x_{i}\left(0\right)$ for all $i\in F$. The relay agent provides state information to explorer $i$ (i.e., \textit{services} explorer $i$) if and only if $\left\Vert y_{i}\left(t\right)-y_{0}\left(t\right)\right\Vert \leq R$ and the communication channel of explorer $i$ is on. We define the \textit{communication switching signal} $\zeta_i$ for explorer $i$ as $\zeta_i=1$ if the communication channel is on for explorer $i$\textcolor{black}{,} and $\zeta_i=0$ if the communication channel is off for explorer $i$. We use $t_{s}^{i}\ge0$ to indicate the $s^{\text{th}}$
servicing instance for explorer $i$. The $(s+1)^{\text{th}}$ servicing instance for explorer $i$ is defined as\footnote{For $s=0,$ $t_{0}^{i}$ is the initial time. For simplicity, we take $t_{0}^{i}=0.$}  
\begin{equation}\nonumber
t_{s+1}^{i}\triangleq \inf\left\{ t > t_{s}^{i}: (\Vert y_{i}\left(t\right)-y_{0}\left(t\right)\Vert \leq R) \land (\zeta_i(t)=1)\right\}, 
\end{equation}
where $\land$ denotes the \textit{conjunction} logical connective.

\subsection{Approximate Consensus}
Given a goal region centered at $Cx_g$ with radius $R_f$, one objective is to design distributed controllers for all explorers that achieve approximate consensus within the goal region. This objective is decomposed into two tasks, where it is the task of the explorers to dead-reckon towards $Cx_g$, and it is the task of the relay agent to intermittently service each explorer. 

Since a single relay agent must intermittently service $Q$ explorers, the MAS can be modeled as a switched system, where the relay agent has $2^Q$ modes of operation, i.e., the relay agent can service no, a single, multiple, or all explorers at an instance depending on the configuration of the explorers. Let $\sigma:\mathbb{T}\to 2^F$ be a piece-wise constant switching signal that determines the mode of operation for the relay agent, where $2^F$ denotes the power set of $F$. The switching signal $\sigma$ also determines the servicing times for all explorers, i.e., $\{t_s^i\}_{s=0}^{\infty}$. To quantify the objective, let the tracking error $e_{i}:\mathbb{T}\rightarrow\mathbb{R}^{m}$ of explorer $i$ be defined as 
\begin{equation}
    e_{i}\left(t\right)\triangleq x_{g}-x_{i}\left(t\right).
\label{E: ei}
\end{equation}
To facilitate the analysis, let the state estimation error $e_{1,i}:\mathbb{T}\to\mathbb{R}^m$ be defined as 
\begin{equation}
e_{1,i}\left(t\right)\triangleq\hat{x}_{i}\left(t\right)-x_{i}\left(t\right),
\label{e1i}
\end{equation}
where $\hat{x}_i:\mathbb{T}\to\mathbb{R}^m$ denotes the state estimate of explorer $i$. For each $i\in F$, the state estimate of explorer $i$ is synchronized between explorer $i$ and the relay agent. Let the estimated tracking error $e_{2,i}:\mathbb{T}\to\mathbb{R}^m$ be defined as 
\begin{equation}
e_{2,i}\left(t\right)\triangleq x_{g}-\hat{x}_{i}\left(t\right).  
\label{e2i}
\end{equation} 
Using \eqref{e1i} and \eqref{e2i}, \eqref{E: ei} can be alternatively expressed as
\begin{equation}
    e_{i}(t)=e_{2,i}(t)+e_{1,i}(t).
\label{E: ei Alternative Form}
\end{equation}
Given the tracking error in \eqref{E: ei}, approximate consensus is achieved within the goal region whenever
\begin{equation*}
    \underset{t\to\infty}{\text{lim sup }} \Vert e_i(t) \Vert \leq \frac{R_f}{S_{\max}(C)} \quad \forall i\in F\cup\{0\},
\end{equation*}
where $S_{\max}(C)\in\mathbb{R}_{> 0}$ denotes the maximum singular value of $C$. 

\subsection{\textcolor{black}{State Observer and Controller Development}}
The state estimate of explorer $i\in F$ is generated by the following model-based observer
\begin{align}
\begin{split}
\dot{\hat{x}}_{i}\left(t\right) & \triangleq  -Ae_{2,i}\left(t\right)+Bu_{i}\left(t\right), \text{ }t\in\left[t_{s}^{i},t_{s+1}^{i}\right),\\
\hat{x}_{i}\left(t_{s}^{i}\right) & \triangleq x_{i}\left(t_{s}^{i}\right), 
\end{split}
\label{Reset}
\end{align}
where the position estimate $\hat{y}_i:\mathbb{T}\to\mathbb{R}^z$ of explorer $i$ is defined as \vspace{-0.05in}
\begin{align}
\begin{split}
\hat{y}_{i}\left(t\right) & \triangleq C\hat{x}_{i}\left(t\right).
\end{split}
\label{Reset_y}
\end{align}
The state estimate $\hat{x}_{i}(t)$ is initialized as $\hat{x}_{i}\left(0\right)=x_{i}\left(0\right)$ for all $i\in F$. Note that at each servicing instance $t_{s}^{i}$, the state estimate $\hat{x}_{i}(t)$ of explorer $i$ is reset to $x_{i}(t)$ as outlined in \eqref{Reset}. The controller of explorer $i$ is defined as
\begin{equation}
u_{i}\left(t\right)\triangleq B^{\textrm{T}}P e_{2,i}\left(t\right), 
\label{explorer Controller}
\end{equation}
where $P\in\mathbb{R}^{m\times m}$ is the positive definite solution to the Algebraic Riccati Equation (ARE) given by  
\begin{equation}
    A^{\textrm{T}}P+PA-2PBB^{\textrm{T}}P+kI_{m}=0_{m\times m}
\label{ARE}
\end{equation}
such that $k>0$ is a user-defined parameter, $I_m$ denotes the $m\times m$ identity matrix, and $0_{m\times m}$ denotes the $m\times m$ zero matrix.  Substituting \eqref{explorer Dynamics} and \eqref{Reset} into the time derivative of \eqref{e1i} yields
\begin{equation}
\begin{aligned}
\dot{e}_{1,i}\left(t\right) & =Ae_{1,i}\left(t\right)-Ax_{g}-d_{i}\left(t\right),\text{ }t\in\left[t_{s}^{i},t_{s+1}^{i}\right),\\
e_{1,i}\left(t_{s}^{i}\right) & =0_{m},
\end{aligned}
\label{e1i Dot closed-loop}
\end{equation}
where $0_m\in\mathbb{R}^m$ denotes the $m$-dimensional zero vector. Substituting \eqref{Reset} and \eqref{explorer Controller} into the time derivative of \eqref{e2i} yields
\begin{equation}
\begin{aligned}
\dot{e}_{2,i}\left(t\right) & =\left(A-BB^{\textrm{T}}P\right)e_{2,i}\left(t\right),\text{ }t\in\left[t_{s}^{i},t_{s+1}^{i}\right),\\
e_{2,i}\left(t_{s}^{i}\right) & =x_{g}-x_{i}\left(t_{s}^{i}\right).
\end{aligned}
\label{e2i Dot closed-loop}
\end{equation}
Substituting \eqref{explorer Dynamics}, \eqref{E: ei Alternative Form}, and \eqref{explorer Controller} into the time derivative of \eqref{E: ei} yields
\begin{equation}
\begin{aligned}
    \dot{e}_{i}(t) &=\left(A-BB^{\textrm{T}}P\right)e_{i}\left(t\right)+BB^{\textrm{T}}Pe_{1,i}\left(t\right) \\
	               &-Ax_{g}-d_{i}\left(t\right).
\end{aligned}
\label{E: ei closed-loop dynamics}
\end{equation}
The state $x_i(t)$ is continuous given the motion model in \eqref{explorer Dynamics}. Hence, \eqref{E: ei} implies $e_i(t)$ is continuous. From \eqref{e1i} and \eqref{Reset}, $e_{1,i}(t)$ is piece-wise continuous. Since the disturbance acting on explorer $i$ is continuous, $e_i(t)$ is continuous, and $e_{1,i}(t)$ is piece-wise continuous, the RHS of \eqref{E: ei closed-loop dynamics} is piece-wise continuous. Hence, $e_i(t)$ is piece-wise continuously differentiable, and therefore, locally Lipschitz. 

\subsection{Metric Temporal Logic (MTL)}
\label{MTL}   
To achieve stability of the switched system and approximate consensus
of the explorers while satisfying the practical constraints
of the relay agent, the requirements of the MAS can be specified in MTL specifications (see details in Section \ref{sec_reactive}). In this subsection, we briefly review MTL interpreted
over discrete-time trajectories~\cite{FainekosMTL}. 
The domain of the position $y$ of a certain agent is denoted by $\mathcal{Y}\subset\mathbb{R}^z$. The Boolean domain is $\mathbb{B} = \{\textrm{True}, \textrm{False}\}$, and the time index set
is $\mathbb{I} = \{0,1,\dots\}$. With slight abuse of notation, we use $y$ to denote the discrete-time trajectory as a function from $\mathbb{I}$ to $\mathcal{Y}$. A set $AP$ is a set of atomic propositions, each of which maps $\mathcal{Y}$ to $\mathbb{B}$. The syntax of MTL is defined recursively as
\[
\phi:=\top\mid \pi\mid\lnot\phi\mid\phi_{1}\wedge\phi_{2}\mid\phi_{1}\vee
\phi_{2}\mid\phi_{1}\mathcal{U}_{\mathcal{I}}\phi_{2}
\]
where $\top$ stands for the Boolean constant True, $\pi\in AP$ is an atomic 
proposition, $\lnot$ (negation), $\wedge$ (conjunction), $\vee$ (disjunction) 
are standard Boolean connectives, $\mathcal{U}$ is a temporal operator
representing \textquotedblleft until\textquotedblright~and $\mathcal{I}$ is a time interval of
the form $\mathcal{I}=[j_{1},j_{2}]$ ($j_1\le j_2$, $j_1, j_2\in\mathbb{I}$). We
can also derive two useful temporal operators from \textquotedblleft
until\textquotedblright~($\mathcal{U}$), which are \textquotedblleft
eventually\textquotedblright~$\Diamond_{\mathcal{I}}\phi\triangleq\top\mathcal{U}_{\mathcal{I}}\phi$ and
\textquotedblleft always\textquotedblright~$\Box_{\mathcal{I}}\phi\triangleq\lnot\Diamond_{\mathcal{I}}\lnot\phi$. We define the set of states that satisfy the atomic proposition $\pi$ as $\mathcal{O}(\pi)\subset \mathcal{Y}$. 

Next, we introduce the Boolean semantics of MTL for trajectories of finite length in the strong and the weak view, which are modified from the literature of temporal logic model checking and monitoring \cite{Eisner2003,KupfermanVardi2001,Ho2014}.  We use $t[j]\in\mathbb{T}$ to denote the time instant at time index $j\in\mathbb{I}$ and $y^j\triangleq y(t[j])$ to denote the value of $y$ at time $t[j]$. In the following, $(y^{0:H},j)\models_{\rm{S}}\phi$ (resp. $(y^{0:H},j)\models_{\rm{W}}\phi$)
means the trajectory $y^{0:H}\triangleq y^0\dots y^H$ $(H\in\mathbb{Z}_{\ge0})$ strongly (resp. weakly) satisfies $\phi$ at time index $j$, $(y^{0:H},j)\not\models_{\rm{S}}\phi$ (resp. $(y^{0:H},j)\not\models_{\rm{W}}\phi$)
means $y^{0:H}$ fails to strongly (resp. weakly) satisfy $\phi$ at time index $j$. 

\begin{definition}
	The Boolean semantics of MTL for trajectories of finite length in the strong view is defined recursively as follows~\cite{zhe_advisory}:
	\[
	\begin{split}
	(y^{0:H},j)\models_{\rm{S}}\pi~\mbox{iff}~& j\le H~\mbox{and}~y^j\in\mathcal{O}(\pi),\\
	(y^{0:H},j)\models_{\rm{S}}\lnot\phi~\mbox{iff}~ & (y^{0:H},j)\not\models_{\rm{W}}\phi,\\
	(y^{0:H},j)\models_{\rm{S}}\phi_{1}\wedge\phi_{2}~\mbox{iff}~ &  (y^{0:H},j)\models_{\rm{S}}\phi
	_{1}~\\& ~\mbox{and}~(y^{0:H},j)\models_{\rm{S}}\phi_{2},\\
	(y^{0:H},j)\models_{\rm{S}}\phi_{1}\mathcal{U}_{\mathcal{I}}\phi_{2}~\mbox{iff}~ &  \exists
	j^{\prime}\in j+\mathcal{I}, \mbox{s.t.} (y^{0:H},j^{\prime})\models_{\rm{S}}\phi_{2},\\
	&   (y^{0:H},j^{\prime\prime})\models_{\rm{S}}\phi_{1} \forall j^{\prime\prime}\in\lbrack j,j^{\prime}).
	\end{split}
	\]
	\label{strong}
\end{definition}

 \begin{definition}
	The Boolean semantics of MTL for trajectories of finite length in the weak view is defined recursively as follows~\cite{zhe_advisory}:
	\begin{align}\nonumber
	\begin{split}
	(y^{0:H},j)\models_{\rm{W}}\pi~&\mbox{iff}~ \textrm{either of the following holds}:\\
	& 1)~j\le H~\mbox{and}~y^j\in\mathcal{O}(\pi); ~2)~j>H,\\	
	(y^{0:H},j)\models_{\rm{W}}\lnot\phi~&\mbox{iff}~  (y^{0:H},j)\not\models_{\rm{S}}\phi,\\
	(y^{0:H},j)\models_{\rm{W}}\phi_{1}\wedge\phi_{2}~&\mbox{iff}~   (y^{0:H},j)\models_{\rm{W}}\phi
	_{1}~\\&~\mbox{and}~(y^{0:H},j)\models_{\rm{W}}\phi_{2},\\	
	(y^{0:H},j)\models_{\rm{W}}\phi_{1}\mathcal{U}_{\mathcal{I}}\phi_{2}~&\mbox{iff}~  \exists
	j^{\prime}\in j+\mathcal{I}, \mbox{s.t.} (y^{0:H},j^{\prime})\models_{\rm{W}}\phi_{2},\\
	&  (y^{0:H},j^{\prime\prime})\models_{\rm{W}}\phi_{1} \forall j^{\prime\prime}\in\lbrack j, j^{\prime}).
	\label{weak}
	\end{split}
	\end{align}
\end{definition}

Intuitively, if a trajectory of finite length can be extended to infinite length, then the strong view indicates that the truth value of the formula on the infinite-length trajectory is already ``determined'' on the trajectory of finite length, while the weak view indicates that it may not be ``determined'' yet \cite{Ho2014}. As an example, a trajectory $y^{0:3}=y^0y^1y^2y^3$ is not possible to strongly satisfy $\phi=\Box_{[0,5]}\pi$ at time 0, but $y^{0:3}$ is possible to strongly violate $\phi$ at time 0, i.e., $(y^{1:3},0)\models_{\rm{S}}\lnot\phi$ is possible.

For an MTL formula $\phi$, the necessary length $L(\phi)$ is defined recursively as follows \cite{Maler2004}: 
\[                                                                   
\begin{split}
&L(\pi) :=0, ~L(\lnot\phi) :=L(\phi),L(\phi_{1}\wedge\phi_{2}):=\max(L(\phi_{1}),L(\phi_{2})),\\
&L(\phi_1\mathcal{U}_{[j_1, j_2]}\phi_{2}):=\max(L(\phi_{1}),L(\phi_{2}))+j_2.
\end{split}
\]   

\subsection{Problem Statement}
\label{sec_problem}
We now present the problem formulation for the control of the MAS with intermittent communication and MTL specifications.

\begin{problem}
	Design the control inputs for the relay agent $\mathbf{u}_0 = [u^0_0, u^1_0, \cdots]$ ($u^j_0$ denotes the control input at time index $j$) such that the following characteristics are satisfied while minimizing the control effort $\sum^{\infty}_{j=0}\norm{u^j_0}$:\\
	\textit{Correctness}: A given MTL specification $\phi$ is weakly satisfied by the trajectory of the relay agent.\\
	\textit{Stability}: The error signal $e_{1,i}\left(t\right)$ is uniformly bounded for each $i\in F$.\\
	\textit{Approximate Consensus}: The states of the explorers reach approximate consensus within the goal region centered at $Cx_g$ with radius $R_f$.
	\label{problem}
\end{problem}

\section{Stability and Consensus Analysis}
\label{sec_dwell_time}
In this section, we provide conditions that generate a stable switched system and enable approximate consensus for the explorers. 

% These conditions take the form of a maximum dwell-time condition for each explorer. Note that the maximum dwell-time condition upper bounds the difference between consecutive servicing instances by the relay agent for each explorer.

To facilitate the stability analysis, we define the following objects. Let $\lambda_{\max}(\mathcal{A})\in\mathbb{R}$ and $\lambda_{\min}(\mathcal{A})\in\mathbb{R}$ denote the maximum and minimum eigenvalue of the symmetric matrix $\mathcal{A}\in\mathbb{R}^{p\times p}$, respectively. Let $\kappa_{i}\triangleq S_{\max}\left(A\right)\overline{x}_{g}+\overline{d}_{i}\in\mathbb{R}_{>0}$, where $\overline{x}_{g}\in\mathbb{R}_{>0}$ is a bounding constant such that $\left\Vert x_{g}\right\Vert \leq\overline{x}_{g}$. Recall that $\overline{d}_{i}$ is an upper bound for the disturbance acting of explorer $i$. Let $f:\mathbb{R}\to\mathbb{R}$ be an essentially bounded measurable function. Then, $f\in\mathcal{L}_{\infty}$ if and only if $\inf\left\{ \bar{f}\in\mathbb{R}_{> 0}:\left|f\left(x\right)\right|\leq \bar{f}\text{ for almost every } x\in\mathbb{R}\right\} \in\mathbb{R}_{\geq0}$. Let $V_{T}\in\mathbb{R}_{>0}$ be a user-defined parameter that quantifies the maximum tolerable state estimation error, i.e., it is desirable to ensure $\left\Vert e_{1,i}\left(t\right)\right\Vert \leq V_{T}$ for all $t\geq0$ and $i\in F$. We now derive a maximum dwell-time condition that ensures $\left\Vert e_{1,i}\left(t\right)\right\Vert \leq V_{T}$ for all $t\in\left[t_{s}^{i},t_{s+1}^{i}\right]$, where continuous satisfaction of the maximum dwell-time condition by the relay agent ensures $\left\Vert e_{1,i}\left(t\right)\right\Vert \leq V_{T}$ for all $t\geq0$.
\textcolor{black}{
\begin{theorem}
	\label{Theorem 1} 
	If $\left\Vert e_{1,i}\left(t_{s}^{i}\right)\right\Vert =0$ and the relay agent satisfies the maximum dwell-time condition given by
    \begin{equation}
        t_{s+1}^{i}-t_{s}^{i} \leq\frac{1}{S_{\max}\left(A\right)}\ln\left(\frac{V_{T}S_{\max}\left(A\right)}{\kappa_{i}}+1\right),
    \label{Maximum Dwell-Time Condition}
    \end{equation}
    then $\left\Vert e_{1,i}\left(t\right)\right\Vert \leq V_{T}$ for all $t\in\left[t_{s}^{i},t_{s+1}^{i}\right]$. 
\end{theorem}}

\begin{proof}
	Let $t\geq t_{s}^{i}$, and suppose $\left\Vert e_{1,i}\left(t_{s}^{i}\right)\right\Vert =0$.\footnote{$\left\Vert e_{1,i}\left(t_{s}^{i}\right)\right\Vert =0$ because the relay agent serviced explorer $i$ at time $t_s^i$.}  Consider the common Lyapunov-like functional candidate $V_{1,i}:\mathbb{R}^{m}\rightarrow\mathbb{R}_{\geq0}$ defined as $V_{1,i}\left(e_{1,i}\left(t\right)\right)\triangleq\frac{1}{2}e_{1,i}^{\textrm{T}}\left(t\right)e_{1,i}\left(t\right)$. Given \eqref{explorer Dynamics} and \eqref{Reset}, \eqref{e1i} is continuously differentiable over $\left[t_{s}^{i},t_{s+1}^{i}\right).$
	Substituting \eqref{e1i Dot closed-loop} into the time derivative of $V_{1,i}$ yields \textcolor{black}{ $\dot{V}_{1,i}\left(e_{1,i}\left(t\right)\right)=e_{1,i}^{\textrm{T}}\left(t\right)\left(Ae_{1,i}\left(t\right)-Ax_{g}-d_{i}\left(t\right)\right),$
	which can be upper bounded by \vspace{-0.05in}
	\begin{equation}
	\dot{V}_{1,i}\left(e_{1,i}\left(t\right)\right)\leq S_{\max}\left(A\right)\left\Vert e_{1,i}\left(t\right)\right\Vert ^{2}+\kappa_{i}\left\Vert e_{1,i}\left(t\right)\right\Vert.
	\label{V1 Dot Bound}
	\end{equation}}
	\noindent Using the definition of $V_{1,i}$, \eqref{V1 Dot Bound} can be upper bounded by $\dot{V}_{1,i}\left(e_{1,i}\left(t\right)\right)\leq 2S_{\max}\left(A\right)V_{1,i}\left(e_{1,i}\left(t\right)\right)+\kappa_{i}\sqrt{2V_{1,i}\left(e_{1,i}\left(t\right)\right)}$. Using the Comparison Lemma \cite[Lemma 3.4]{Khalil} over $\left[t_{s}^{i},t_{s+1}^{i}\right)$, \vspace{-0.05in}
	\begin{equation}
	V_{1,i}\left(e_{1,i}\left(t\right)\right)\leq\left(\frac{\kappa_{i}}{S_{\max}\left(A\right)}\frac{\sqrt{2}}{2}\left(e^{S_{\max}\left(A\right)\left(t-t_{s}^{i}\right)}-1\right)\right)^{2}.
	\label{V1 Bound}
	\end{equation}
	Substituting the definition of $V_{1,i}$ into \eqref{V1 Bound} yields $\left\Vert e_{1,i}\left(t\right)\right\Vert \leq\frac{\kappa_{i}}{S_{\max}\left(A\right)}\left(e^{S_{\max}\left(A\right)\left(t-t_{s}^{i}\right)}-1\right)$.
	Define $\Phi_{i}:\left[t_{s}^{i},t_{s+1}^{i}\right)\rightarrow\mathbb{R}_{\geq0}$ as
	\begin{equation}
		\Phi_{i}\left(t\right)\triangleq\frac{\kappa_{i}}{S_{\max}\left(A\right)}\left(e^{S_{\max}\left(A\right)\left(t-t_{s}^{i}\right)}-1\right).
	\label{Phi_i}
	\end{equation}
	Since $\left\Vert e_{1,i}\left(t\right)\right\Vert \leq\frac{\kappa_{i}}{S_{\max}\left(A\right)}\left(e^{S_{\max}\left(A\right)\left(t-t_{s}^{i}\right)}-1\right)$
	for all $t\in\left[t_{s}^{i},t_{s+1}^{i}\right)$ and $\left\Vert e_{1,i}\left(t_{s+1}^{i}\right)\right\Vert =0$,
	where $t_{s+1}^i>t_s^i$ and $\Phi_{i}\left(t_{s+1}^{i}\right)>0,$ we see that $\left\Vert e_{1,i}\left(t\right)\right\Vert \leq\Phi_{i}\left(t\right)$
	for all $t\in\left[t_{s}^{i},t_{s+1}^{i}\right].$ If $\Phi_{i}\left(t_{s+1}^{i}\right)\leq V_{T},$
	then $\left\Vert e_{1,i}\left(t\right)\right\Vert \leq V_{T}$ for
	all $t\in\left[t_{s}^{i},t_{s+1}^{i}\right].$ Moreover, $\Phi_{i}\left(t_{s+1}^{i}\right)\leq V_{T}$ yields the dwell-time condition in \eqref{Maximum Dwell-Time Condition}. Hence, $\left\Vert e_{1,i}\left(t\right)\right\Vert \leq V_{T}$ for all $t\in\left[t_{s}^{i},t_{s+1}^{i}\right]$ provided $\left\Vert e_{1,i}\left(t_{s}^{i}\right)\right\Vert =0$ and \eqref{Maximum Dwell-Time Condition} hold. 
\end{proof}
\textcolor{black}{
Next, we show that the observer in \eqref{Reset} ensures the estimated tracking error in \eqref{e2i} is exponentially regulated for all $t\in\left[t_{s}^{i},t_{s+1}^{i}\right)$ and each servicing instance $s\in\mathbb{Z}_{\geq0}$. }

\begin{theorem}
\textcolor{black}{
If the ARE in \eqref{ARE} is satisfied, then the observer in \eqref{Reset} and controller in \eqref{explorer Controller} ensure the estimated tracking error in \eqref{e2i} is exponentially regulated in the sense that \vspace{-0.03in}
\begin{equation}
    \left\Vert e_{2,i}\left(t\right)\right\Vert \leq\sqrt{\frac{\lambda_{\max}\left(P\right)}{\lambda_{\min}\left(P\right)}}\left\Vert e_{2,i}\left(t_{s}^{i}\right)\right\Vert e^{-\frac{k}{2\lambda_{\max}\left(P\right)}\left(t-t_{s}^{i}\right)}
\label{Norm of e2i Bound}
\end{equation}
for all $t\in\left[t_{s}^{i},t_{s+1}^{i}\right)$ and each servicing instance $s\in\mathbb{Z}_{\geq0}$. 
\label{Theorem 2}}
\end{theorem}

\begin{proof}
\textcolor{black}{
Consider the common Lyapunov functional $V_{2,i}:\mathbb{R}^{m}\rightarrow\mathbb{R}_{\geq0}$ defined as $ V_{2,i}\left(e_{2,i}\left(t\right)\right)	\triangleq e_{2,i}^{\textrm{T}}\left(t\right)Pe_{2,i}\left(t\right)$.
By the Rayleigh quotient, it follows that \vspace{-0.05in}
\begin{equation}
    \lambda_{\min}\left(P\right)\left\Vert e_{2,i}\left(t\right)\right\Vert ^{2}\leq V_{2,i}\left(e_{2,i}\left(t\right)\right)\leq\lambda_{\max}\left(P\right)\left\Vert e_{2,i}\left(t\right)\right\Vert ^{2}.
\label{V2i Rayleigh Bounds}
\end{equation}
By \eqref{Reset}, \eqref{e2i} is continuously differentiable over $\left[t_{s}^{i},t_{s+1}^{i}\right)$. Substituting \eqref{e2i Dot closed-loop} into the time derivative of $V_{2,i}$ yields \vspace{-0.05in}
\begin{equation}
    \dot{V}_{2,i}\left(e_{2,i}\left(t\right)\right)=e_{2,i}^{\textrm{T}}\left(t\right)\left(A^{\textrm{T}}P+PA-2PBB^{\textrm{T}}P\right)e_{2,i}\left(t\right).
\label{V2i Dot}
\end{equation}
Using the ARE in \eqref{ARE}, it follows that \eqref{V2i Dot} is equivalent to $\dot{V}_{2,i}\left(e_{2,i}\left(t\right)\right)=-k\left\Vert e_{2,i}\left(t\right)\right\Vert ^{2}$.
Using \eqref{V2i Rayleigh Bounds}, we then see that \vspace{-0.05in}
\begin{equation}
    \dot{V}_{2,i}\left(e_{2,i}\left(t\right)\right)\leq-\frac{k}{\lambda_{\max}\left(P\right)}V_{2,i}\left(e_{2,i}\left(t\right)\right).
\label{V2i Dot Bound 2}
\end{equation}
Invoking the Comparison Lemma in \cite[Lemma 3.4]{Khalil} on \eqref{V2i Dot Bound 2} over $\left[t_{s}^{i},t_{s+1}^{i}\right)$ yields \vspace{-0.05in}
\begin{equation}
    V_{2,i}\left(e_{2,i}\left(t\right)\right)\leq V_{2,i}\left(e_{2,i}\left(t_{s}^{i}\right)\right)e^{-\frac{k}{\lambda_{\max}\left(P\right)}\left(t-t_{s}^{i}\right)},
\label{V2i Bound}
\end{equation}
where substituting the definition of $V_{2,i}$ and \eqref{V2i Rayleigh Bounds} into \eqref{V2i Bound} yields \eqref{Norm of e2i Bound}.
}
\end{proof}
We now show the tracking error in \eqref{E: ei} is uniformly ultimately bounded (UUB).
\vspace{+0.2in}
\begin{theorem}
If the relay agent satisfies the maximum dwell-time condition in \eqref{Maximum Dwell-Time Condition} for each $s\in\mathbb{Z}_{\geq0}$ and $e_{1,i}\left(t_{0}^{i}\right)=0_{m}$, then the observer in \eqref{Reset} and controller in \eqref{explorer Controller} ensure the tracking error in \eqref{E: ei} is uniformly ultimately bounded in the sense that \vspace{-0.05in}
\begin{equation}
\begin{aligned}
\left\Vert e_{i}\left(t\right)\right\Vert & \leq\frac{\lambda_{\max}\left(P\right)\rho}{\lambda_{\min}\left(P\right)k}\left(1-e^{-\frac{k}{2\lambda_{\max}\left(P\right)}t}\right)\\
&+\sqrt{\frac{\lambda_{\max}\left(P\right)}{\lambda_{\min}\left(P\right)}}\left\Vert e_{i}\left(0\right)\right\Vert e^{-\frac{k}{2\lambda_{\max}\left(P\right)}t},
\end{aligned}
\label{ei UUB Bound}
\end{equation}
where $k>0$ is a user-defined gain used in \eqref{ARE} and $\rho\triangleq2\overline{d}_{i}S_{\max}\left(P\right)+2V_{T}S_{\max}\left(PBB^{\textrm{T}}P\right)+2S_{\max}\left(PA\right)\overline{x}_{g}\in\mathbb{R}_{>0}$. 
\label{Theorem 3} 
\end{theorem}	
\vspace{-0.2in}
\begin{proof} 
Suppose the relay agent satisfies the maximum dwell-time condition in \eqref{Maximum Dwell-Time Condition} for each $s\in\mathbb{Z}_{\geq0}$ and $e_{1,i}\left(t_{0}^{i}\right)=0_{m}$. Consider the common Lyapunov functional $V_{i}:\mathbb{R}^{m}\rightarrow\mathbb{R}_{\geq0}$ defined by $ V_{i}\left(e_{i}\left(t\right)\right)	\triangleq e_{i}^{\textrm{T}}\left(t\right)Pe_{i}\left(t\right)$. Recall that $e_i(t)$ is continuous, and $\dot{e}_i(t)$ is piece-wise continuous, where the discontinuities occur at each servicing instance. Hence, the set of discontinuities is countable. By the Rayleigh quotient, it follows that  \vspace{-0.05in}
\begin{equation}
    \lambda_{\min}\left(P\right)\left\Vert e_{i}\left(t\right)\right\Vert ^{2}	\leq V_{i}\left(e_{i}\left(t\right)\right)\leq\lambda_{\max}\left(P\right)\left\Vert e_{i}\left(t\right)\right\Vert ^{2}.
\label{Vi Rayleigh Bounds}
\end{equation}
Substituting \eqref{E: ei closed-loop dynamics} into the time derivative of $V_{i}$ yields \vspace{-0.05in}
\begin{equation}
\begin{aligned}
    \dot{V}_{i}\left(e_{i}\left(t\right)\right)&=e_{i}^{\textrm{T}}\left(t\right)\left(A^{\textrm{T}}P+PA-2PBB^{\textrm{T}}P\right)e_{i}\left(t\right)\\
    &+2e_{i}^{\textrm{T}}\left(t\right)P\left(BB^{\textrm{T}}Pe_{1,i}\left(t\right)-Ax_{g}-d_{i}\left(t\right)\right).
\end{aligned}
\label{Vi Dot 1}
\end{equation}
Using the ARE in \eqref{ARE}, it follows that \eqref{Vi Dot 1} can be upper bounded as
\begin{equation}
\begin{aligned}
    \dot{V}_{i}\left(e_{i}\left(t\right)\right)&\leq-k\left\Vert e_{i}\left(t\right)\right\Vert ^{2}+2S_{\max}\left(P\right)\left\Vert e_{i}\left(t\right)\right\Vert \left\Vert d_{i}\left(t\right)\right\Vert \\
    &+2S_{\max}\left(PBB^{\textrm{T}}P\right)\left\Vert e_{i}\left(t\right)\right\Vert \left\Vert e_{1,i}\left(t\right)\right\Vert \\
    &+2S_{\max}\left(PA\right)\left\Vert e_{i}\left(t\right)\right\Vert \left\Vert x_{g}\right\Vert.
\end{aligned}
\label{Vi Dot Bound 1}
\end{equation}
Since the relay agent satisfies the maximum dwell-time condition in \eqref{Maximum Dwell-Time Condition} for each $s\in\mathbb{Z}_{\geq0}$, $\left\Vert e_{1,i}\left(t\right)\right\Vert \leq V_{T}$ for all $t\geq0$ by Theorem \ref{Theorem 1}. Recall that $\left\Vert d_{i}\left(t\right)\right\Vert \leq\overline{d}_{i}$ and that $\left\Vert x_{g}\right\Vert \leq\overline{x}_{g}$. Hence, \eqref{Vi Dot Bound 1} can be upper bounded as 
\begin{equation}
    \dot{V}_{i}\left(e_{i}\left(t\right)\right)\leq-k\left\Vert e_{i}\left(t\right)\right\Vert ^{2}+\rho\left\Vert e_{i}\left(t\right)\right\Vert,
\label{Vi Dot Bound 2}
\end{equation} 
where the auxiliary constant $\rho$ is defined in Theorem \ref{Theorem 3}. Substituting \eqref{Vi Rayleigh Bounds} into \eqref{Vi Dot Bound 2} and integrating both sides of the resulting inequality over $[0,\infty)$ yields \eqref{ei UUB Bound}. Observe that \eqref{ei UUB Bound} implies $e_{i}\left(t\right)\in\mathcal{L}_{\infty}$. Since $e_{i}\left(t\right)\in\mathcal{L}_{\infty}$ and $e_{1,i}\left(t\right)\in\mathcal{L}_{\infty}$ given the relay agent satisfies the maximum dwell-time condition in \eqref{Maximum Dwell-Time Condition} for each $s\in\mathbb{Z}_{\geq0}$, \eqref{E: ei Alternative Form} implies $e_{2,i}\left(t\right)\in\mathcal{L}_{\infty}$. Hence, $u_{i}\left(t\right)\in\mathcal{L}_{\infty}$ given \eqref{explorer Controller} and $e_{2,i}\left(t\right)\in\mathcal{L}_{\infty}$.
\end{proof} 
\begin{remark}
From \eqref{ei UUB Bound}, we see that
\begin{equation*}
    \underset{t\to\infty}{\text{lim sup}} \ \Vert e_i(t) \Vert \leq \frac{\lambda_{\max}\left(P\right)\rho}{\lambda_{\min}\left(P\right)k}\triangleq \Lambda(\rho),
\end{equation*}
where $\Lambda(\rho)$ can be made small by making $\rho$ small, i.e., selecting a small $V_{T}\in\mathbb{R}_{>0}$ and setting the desired state as the origin. A change of coordinate transformation can be used to make the desired state the origin.
\label{Remark 1}
\end{remark}
\begin{remark}
Note that $\left\Vert Cx_g - y_{i}\left(t\right)\right\Vert \leq S_{\max}\left(C\right)\left\Vert e_{i}\left(t\right)\right\Vert$. If the radius of the goal region is selected such that $\Lambda(\rho)S_{\max}(C)< R_{f}$, then $e_{1,i}\left(t\right)=0_{m}$ provided $\left\Vert e_{i}\left(t\right)\right\Vert \leq \Lambda(\rho)$. Since $e_{1,i}\left(t\right)=0_{m}$ within the goal region, $\Lambda(\rho)$ can be reduced to $\Lambda(\rho^{\ast})$, where $\rho^{\ast}\triangleq2\overline{d}_{i}S_{\max}\left(P\right)+2S_{\max}\left(PA\right)\bar{x}_{g}$. 
\label{Remark 2}
\end{remark}

\section{Controller Synthesis with Intermittent Communication and MTL Specifications}
\label{sec_reactive}
In this section, we provide the framework and algorithms for controller synthesis of the relay agent to satisfy the maximum dwell-time conditions and the practical constraints. The controller synthesis for the relay agent is conducted iteratively as the state estimates for the explorers are reset to the true state values whenever they are serviced by the relay agent, and thus the control inputs need to be recomputed with the reset values.

% We assume that the communication is only possible at discrete time instants, with $T_s$ time periods apart and controlled by the communication switching signal $\zeta_i$. 

We define the discrete time set $\mathbb{T}_d\triangleq\{t[0], t[1], \dots\}$, where $t[j]=jT_s$ for $j\in\mathbb{I}$, and $T_s\in\mathbb{R}_{>0}$ is the sampling period. The maximum dwell-time $\frac{1}{S_{\textrm{max}}\left(A\right)}\ln\left(\frac{S_{\textrm{max}}\left(A\right)V_T}{\kappa_{i}}+1\right)$ in (\ref{Maximum Dwell-Time Condition}) for explorer $i$ $(i\in F)$ is in the interval $[n_iT_s, (n_i+1)T_s)$ for some non-negative integer $n_i$. 
We use the following MTL specifications for encoding the maximum dwell-time condition ($\eta\in[0, R)$ is a user-defined parameter):
\begin{align}
\begin{split}
&\phi_{\textrm{m}}=\bigwedge_{1\le i\le Q}\big(\Box\Diamond_{[0, n_i]}\norm{y_0-\hat{y}_i}\le\eta\big),
\end{split}
\end{align} 
where $\phi_{\textrm{m}}$ means ``for any explorer $i$, the relay agent needs to be within $\eta$ distance from the estimated position of explorer $i$ at least once in any $n_iT_s$ time periods''. 

The relay agent also needs to satisfy an MTL specification $\phi_{\textrm{p}}$ for the practical constraints. One example of $\phi_{\textrm{p}}$ is as follows.
\begin{align}
&\phi_{\textrm{p}}=\Box\Diamond_{[0, c]}\big((y_0\in G_1)\vee (y_0\in G_2)\big)\wedge\Box (y_0\in D),
\end{align} 
which means ``the relay agent needs to reach the charging station $G_1$ or $G_2$ at least once in any $cT_s$ time periods, and it should always remain in the region $D$'' ($c$ is a positive integer).

Combining $\phi_{\textrm{m}}$ and $\phi_{\textrm{p}}$, the MTL specification for the relay agent is                                  
$\phi=\phi_{\textrm{m}}\wedge\phi_{\textrm{p}}$. We use $[\phi]^{\ell}_{j}$ to denote the formula modified from the MTL formula $\phi$ when $\phi$ is evaluated at time index $j$ and the current time index is $\ell$.  $[\phi]^{\ell}_{j}$ can be calculated recursively as follows (we use $\pi_{j}$ to denote the atomic predicate $\pi$ evaluated at time index $j$):
\begin{align}
\begin{split}                     
[\pi]^{\ell}_{j} :=&  
\begin{cases}
\pi_{j},& \mbox{if $j>\ell$}\\  	
\top,& \mbox{if $j\le \ell$ and $y^{j}\in\mathcal{O}(\pi)$}\\  
\bot,& \mbox{if $j\le \ell$ and $y^{j}\not\in\mathcal{O}(\pi)$}                                                                                                                                         
\end{cases}\\
[\neg\phi]^{\ell}_{j}  :=&\neg[\phi]^{\ell}_{j}\\
[\phi_1\wedge\phi_2]^{\ell}_{j}:=&[\phi_1]^{\ell}_{j}\wedge[\phi_2]^{\ell}_{j}\\
[\phi_1\mathcal{U}_{\mathcal{I}}\phi_2]^{\ell}_{j} :=&\bigvee_{j'\in (j+\mathcal{I})}\Big([\phi_2]^{\ell}_{j'}\wedge\bigwedge_{j\le j''<j'}[\phi_1]^{\ell}_{j''}\Big),
\end{split}
\label{update_phi}
\end{align}
where $\bot$ stands for the Boolean constant False.
If the MTL formula $\phi$ is evaluated at the initial time index (which is the usual case when the task starts at the initial time), then the modified formula is $[\phi]^{\ell}_{0}$. 

Algorithm 1 shows the controller synthesis approach with intermittent communication and MTL specifications. 
The controller synthesis problem can be formulated as a sequence of mixed integer linear programming (MILP) problems, denoted as MILP-sol in Line \ref{MILP-sol}, and expressed as follows:
\begin{align}
\underset{\mathbf{u}^{\ell:\ell+N-1}_0}{\argmin} ~ & \sum_{j=\ell}^{\ell+N-1}\norm{u_0^j}   
	\end{align}
	\begin{align}
\text{subject to:} ~ 
& x_0^{j+1}=\bar{A}_0x_0^{j}+\bar{B}^0u_0^j, ~y_0^{j}=\bar{C}_0x_0^{j}, \nonumber \\
	& ~~~\forall i\in F, \forall j=\ell,\dots,\ell+N-1,\\
& \hat{x}_i^{j+1}=\bar{A}(\hat{x}_i^{j}-x_g)+\bar{B}u_i^j, ~\hat{y}_i^{j}=\bar{C}\hat{x}_i^{j}, \nonumber \\ & ~~~ \forall i\in F, \forall j=\ell,\dots,\ell+N-1, \label{update_constraint}\\
& u_{0,\textrm{min}}\le u_0^j\le u_{0,\textrm{max}}, \forall i\in F, \nonumber \\ & ~~~~~~~~~~~~~~\forall j=\ell,\dots,\ell+N,\\
& (\tilde{y}^{\ell:\ell+N-1}, 0)\models_{\textrm{W}} [\phi]^{\ell}_{0}, 
\label{MILP}               
\end{align} 
where the time index $\ell$ is initially set as 0, $N\in\mathbb{Z}_{>0}$ is the number of time instants in the control horizon, $\tilde{y}^{\ell:\ell+N-1}=[y_0^{\ell:\ell+N-1}, \hat{y}_1^{\ell:\ell+N-1}, \dots, \hat{y}_Q^{\ell:\ell+N-1}]$, \(\mathbf{u}^{\ell:\ell+N-1}_0 = [u^{\ell}_0, u^{\ell+1}_0, \cdots, u^{\ell+N-1}_0]\) is the control inputs of the relay agent, the input values are constrained to $[u_{0, \textrm{min}}, u_{0, \textrm{max}}]$, $\bar{A}^0$, $\bar{B}^0$, $\bar{C}^0$, $\bar{A}$, $\bar{B}$ and $\bar{C}$ are converted from $A^0$, $B^0$, $C^0$, $A$, $B$ and $C$ respectively for the discrete-time state-space representation, and $u_i^j$ are control inputs of explorer $i$ from (\ref{explorer Controller}). Note that we only require the trajectory $y_0^{\ell:\ell+N-1}$ to weakly satisfy $\phi$ as $\ell+N-1$ may be less than the necessary length $L(\phi)$.

At each time index $\ell$, we check if there exists any explorer that is being serviced (Line \ref{check_serve}). If there are such explorers, we update the state estimates of those explorers with their true state values (Line \ref{update}). Then, we modify the MTL formula as in (\ref{update_phi}). The MILP is solved for time index $\ell$ with the updated state values and the modified MTL formula $[\phi]^{\ell}_{0}$ (Line \ref{recompute}). The previously computed relay agent control inputs are replaced by the newly computed control inputs from time index $\ell$ to $\ell+N-1$ (Line \ref{replace}).

\begin{algorithm}[h!]
	\caption{Controller synthesis of MASs with intermittent communication and MTL specifications.}                                                                   
	\label{MTLalg}
	\begin{algorithmic}[1]
		\State \textbf{Inputs:}  $x_0^{0}$, $x_i^{0}$, $\phi$, $x_g$, $R_f$, $V_{T}$, $\eta$
		\State $\ell\gets0$, $\ell^{\ast}\gets0$
%		\State \begin{align*}
%		\text{MILP-sol} := \left \{ \begin{aligned} 
%		& \underset{\mathbf{u}^{\ell}_0}{\argmin} ~  
%		J(\mathbf{u}^{\ell}_0)\\
%		&\text{s.t. ~ constraints (26)-(29)}
%		\end{aligned} \right.
%		\end{align*} 
		\State Solve MILP-sol to obtain the optimal inputs $u_{0}^{\ast q}~(q=0,1,\dots,N-1)$ \label{MILP-sol}
		\While{$\norm{Cx_g-y_i(t[\ell])}>R_f$ for some $i\in F$}
		\State $\mathcal{W}=\{i~\vert~\norm{y_0-\hat{y}_i(t[\ell])}\le\eta\}$ \label{check_serve}
		\If{$\mathcal{W}\neq\emptyset$ or $\ell\ge\ell^{\ast}+N$}
		\State 	$\forall i\in\mathcal{W}$, update $\hat{x}_i^{\ell}$ in constraint (\ref{update_constraint}) and change constraint (\ref{update_constraint}) as follows: \label{update}
		\[
		\begin{split}
		& \hat{x}_i^{j+1}=\bar{A}(\hat{x}_i^{j}-x_g)+\bar{B}u_i^j, \forall i\in F, \nonumber \\ & ~~~~~~~~~~~~~~~~~~~~~~~~\forall j=\ell,\ell+1,\dots,\ell+N-1,\\
		& \hat{x}_i^{\ell}=x_i^{\ell}, \forall i\in\mathcal{W}
		\end{split}
		\]	
		\State Re-solve MILP-sol to obtain the optimal inputs $u^{\ast\ell+q}~(q=0,1,\dots, N-1)$	\label{recompute}
		\State $u_{0}^{\ast\ell+q}\gets u^{\ast\ell+q}~(q=0,1,\dots, N-1)$, $\ell^{\ast}\gets\ell$				\label{replace}
		\EndIf	
		\State $\ell\gets\ell+1$
		\EndWhile    		                
		\State Return $\mathbf{u}^{\ast}_0=[u^{\ast0}_0, u^{\ast1}_0, \dots]$
			
	\end{algorithmic}
\end{algorithm}

We use $\hat{t}_{s+1}^{i}$ to denote the $(s+1)^\textrm{th}$ time that $\norm{y_0(t)-y_i(t)}\le\eta$ holds in the discrete time set $\mathbb{T}_d$ for explorer $i$\footnote{For $s=0,$ $\hat{t}_{0}^{i}$ is the initial time, i.e., $\hat{t}_{0}^{i}=0.$}, i.e.,
\begin{align}\nonumber
\hat{t}_{s+1}^{i}\triangleq &\inf\left\{ t> \hat{t}_{s}^{i}: (t\in\mathbb{T}_d)\wedge\big(\left\Vert \hat{y}_{i}\left(t\right)-y_{0}\left(t\right)\right\Vert \leq \eta\big)\right\}.
\end{align}
We design the communication switching signal $\zeta_i$ as follows:
\begin{align}
\zeta_i(t)=&
\begin{cases}
1, ~~~~~~\mbox{if}~t=\hat{t}_{s}^{i}~\mbox{for~some}~s;   \\
0, ~~~~~~\mbox{otherwise}.
\end{cases}
\label{communication} 
\end{align}

Finally, we present Theorem \ref{Theorem 4}, which provides theoretical guarantees for achieving correctness, stability and approximate consensus (in Problem \ref{problem}).
\begin{theorem}
    With the observers in (\ref{Reset}),
	 the controllers for the explorers in \eqref{explorer Controller}, communication switching signal in (\ref{communication}), if each optimization is feasible in Algorithm \ref{MTLalg}, $V_T\in\Big(0, \frac{R-\eta}{S_{\textrm{max}}(C)}\Big]$, $\eta\in[0, R)$, and $\Lambda(\rho)S_{\max}\left(C\right)< R_{f}$,
	then Algorithm \ref{MTLalg} terminates within finite time, the MTL specification $\phi$ is weakly satisfied and the 
	explorers reach approximate consensus within the goal region in the sense that $\underset{t\to\infty}{\text{lim sup }}\left\Vert e_{i}\left(t\right)\right\Vert \leq \Lambda(\rho^{\ast})$, where $\rho^{\ast}=2\overline{d}_{i}S_{\max}\left(P\right)+2S_{\max}\left(PA\right)\bar{x}_{g}$.
	\label{Theorem 4} 
\end{theorem}

\begin{proof}
We first use induction to prove that $\hat{t}_{s}^{i}=t_{s}^{i}$ for each $i$ and $s$. For each $i$, if $s=0$, then $\hat{t}_{0}^{i}=t_{0}^{i}=0$. Now fix $i\in F$ and assume that $\hat{t}_{s}^{i}=t_{s}^{i}$ for some $s\in\mathbb{Z}_{\geq 0}$. We now show that $\hat{t}_{s+1}^{i}=t_{s+1}^{i}$. If each optimization is feasible in Algorithm \ref{MTLalg}, then $\hat{t}_{s+1}^{i}-\hat{t}_{s}^{i}=\hat{t}_{s+1}^{i}-t_{s}^{i}\le n_iT_s \leq\frac{1}{S_{\textrm{max}}\left(A\right)}\ln\left(\frac{S_{\textrm{max}}\left(A\right)V_T}{\kappa_{i}}+1\right)$. Then, following the analysis in the proof of Theorem \ref{Theorem 1}, $\left\Vert e_{1,i}\left(\hat{t}_{s+1}^{i}\right)\right\Vert \leq V_T$. 
Thus, we have $\left\Vert y_{i}\left(\hat{t}_{s+1}^{i}\right)-y_{0}\left(\hat{t}_{s+1}^{i}\right)\right\Vert \leq \left\Vert Cx_{i}\left(\hat{t}_{s+1}^{i}\right)-C\hat{x}_{i}\left(\hat{t}_{s+1}^{i}\right)\right\Vert + \left\Vert C\hat{x}_{i}\left(\hat{t}_{s+1}^{i}\right)-Cx_{0}\left(\hat{t}_{s+1}^{i}\right)\right\Vert\le S_{\textrm{max}}(C)V_T+\eta$. Therefore, if $V_T\le\frac{R-\eta}{S_{\textrm{max}}(C)}$, then $\left\Vert y_{i}\left(\hat{t}_{s+1}^{i}\right)-y_{0}\left(\hat{t}_{s+1}^{i}\right)\right\Vert \leq R$. According to the communication switching signals in (\ref{communication}), $\zeta_i(\hat{t}_{s+1}^{i})=1$. Thus, from the definition of $t_{s+1}^{i}$ in Section \ref{sec_sensing}, $\hat{t}_{s+1}^{i}=t_{s+1}^{i}$ holds. Therefore, $\hat{t}_{s}^{i}=t_{s}^{i}$ for each $i\in F$ and $s\in\mathbb{Z}_{\geq 0}$ by mathematical induction. 
	
If each optimization is feasible in Algorithm \ref{MTLalg}, then the MTL specification $\phi$ is weakly satisfied. With $\hat{t}_{s}^{i}=t_{s}^{i}$, the maximum dwell-time condition in (\ref{Maximum Dwell-Time Condition}) is satisfied for all $t_{s}^{i}$ and $i\in F$. From Theorem \ref{Theorem 3} and Remarks \ref{Remark 1} and \ref{Remark 2}, if $\Lambda(\rho)S_{\max}\left(C\right) < R_{f}$ holds, then, for each $i\in F$, there exists a time $T_i>0$ such that explorer $i$ will be inside the goal region for $t\ge T_{i}$. Thus, at time $\tilde{t}=\max_{i\in F}\{T_i\}$, $\norm{Cx_g-y_i(\tilde{t})}< R_f$ holds for any $i\in F$, i.e., Algorithm \ref{MTLalg} is guaranteed to terminate within finite time.                                           
Finally, if $V_T\in\Big(0, \frac{R-\eta}{S_{\textrm{max}}(C)}\Big]$, then according to Theorem \ref{Theorem 3} and Remark \ref{Remark 2}, we have $\underset{t\to\infty}{\text{lim sup }}\left\Vert e_{i}\left(t\right)\right\Vert \leq \Lambda(\rho^{\ast})$.
\end{proof}
                                                                                                                                  
\section{Implementation}
\label{sec_implementation}
We now demonstrate the controller synthesis approach on the example in Fig. \ref{fig_intro} (in Section \ref{sec_intro}). The relay agent is a quadrotor modeled as a 3-D six degrees of freedom (6-DOF) rigid body \cite{zhe_advisory}. We denote the system state as $x^0_{\rm{q}}=[p_{\rm{q}},\dot{p}_{\rm{q}},\theta_{\rm{q}},\Omega_{\rm{q}}]^{\textrm{T}}\in\mathbb{R}^{12}$, where $p_{\rm{q}}=[x_{\rm{q},1},x_{\rm{q},2},x_{\rm{q},3}]^{\textrm{T}}$ and $\dot{p}_{\rm{q}}=[\dot{x}_{\rm{q},1},\dot{x}_{\rm{q},2},\dot{x}_{\rm{q},3}]^{\textrm{T}}$ are the position and velocity vectors of the quadrotor. The vector $\theta_{\rm{q}}=[\alpha_{\rm{q}},\beta_{\rm{q}},\gamma_{\rm{q}}]^{\textrm{T}}\in\mathbb{R}^{3}$ includes the roll, pitch and yaw Euler angles of the quadrotor. The vector $\Omega_{\rm{q}}\in\mathbb{R}^{3}$ includes the angular velocities rotating around its body frame axes. The general nonlinear dynamic model of such a quadrotor is given by                              
\begin{align}\label{eqn_6DOFdynamics}       
\begin{split}                              
m_{\rm{q}}\ddot{p}_{\rm{q}}=& r(\theta_{\rm{q}})T_{\rm{q}}\mathbf{e}_3-mg\mathbf{e}_3,\nonumber
\end{split} 
\end{align}
\begin{align}      
\begin{split}
\dot{\theta}_{\rm{q}}=&H(\theta_{\rm{q}})\Omega_{\rm{q}},\\
I\dot{\Omega}_{\rm{q}}=&-\Omega_{\rm{q}}\times I\Omega_{\rm{q}}+\tau_{\rm{q}},
\end{split} 
\end{align}
where $m_{\rm{q}}$ is the mass, $g$ is the gravitational acceleration, $I$ is the inertia matrix, $r(\theta_{\rm{q}})$ is the rotation matrix representing the body frame with respect to the inertia frame (which is a function of the Euler angles), $H(\theta_{\rm{q}})$ is the nonlinear mapping matrix that projects the angular velocity $\Omega_{\rm{q}}$ to the Euler angle rate $\dot{\theta}_{\rm{q}}$, $\mathbf{e}_3=[0,0,1]^{\textrm{T}}$, $T_{\rm{q}}$ is the thrust of the quadrotor, and $\tau_{\rm{q}}\in\mathbb{R}^3$ is the torque on the three axes. The control input is $u_0=[u_{0, 1},u_{0, 2},u_{0, 3},u_{0, 4}]^{\textrm{T}}$, where $u_{0, 1}$ is the vertical velocity command, $u_{0, 2},u_{0, 3}$ and $u_{0, 4}$ are the angular velocity commands around its three body axes. By adopting the small-angle assumption and then linearizing the dynamic model around the hover state, a linear kinematic model can be obtained as
\begin{equation}
\begin{array}{ll}
\dot{x}_0=A_0x+B_0u_0, 
\end{array}
\end{equation}
where $x_0=[x_{\rm{q},1}, x_{\rm{q},2}, x_{\rm{q},3}, \dot{x}_{\rm{q},1}, \dot{x}_{\rm{q},2}, \alpha_{\rm{q}},\beta_{\rm{q}},\gamma_{\rm{q}}]^{\textrm{T}}$ is the state of the kinematic model of the quadrotor (relay agent), $A_0\in\mathbb{R}^{8\times8}$, and $B_0\in\mathbb{R}^{8\times4}$. For the 3-D position representation, $y_0=[x_{\rm{q},1}, x_{\rm{q},2}, x_{\rm{q},3}]^{\textrm{T}}$.

The dynamics of the differential drive of explorer $i$ can be feedback linearized into the following equations (see Section V of \cite{zhe_advisory}):
\begin{align}
\begin{bmatrix}
\ddot{x}_{i,1} \\ \ddot{x}_{i,2}
\end{bmatrix} =
\begin{bmatrix}
u_{i,1} \\ u_{i,2}
\end{bmatrix},
\end{align}
where $u_{i,1}$ and $u_{i,2}$
are the inputs in the 2-D plane.
% For consensus, we consider the following control law as in (\ref{explorer Control}):
% \[
% u_{i}\left(t\right)\triangleq k_{i}e_{i, 2}\left(t\right),
% \]
% where $k_1=0.1$, $k_2=0.15$ and $k_3=0.2$, respectively.

% For the state space representation, we choose $x=[w_1, w_2, \dot{w}_1,  \dot{w}_2]^{\textrm{T}}$ and $y=[w_1, w_2]^{\textrm{T}}$.
For the state space representation, $x_i=[x_{i,1}, x_{i,2}]^{\textrm{T}}$, $u_i=[u_{i,1}, u_{i,2}]^{\textrm{T}}$, $d_i=[d_{i,1}, d_{i,2}]^{\textrm{T}}$ and $y_i=[x_{i,1}, x_{i,2}, 0]^{\textrm{T}}$. The initial 3-D positions of the three explorers are $[-100,-100,0]^{\textrm{T}}$, $[100,150,0]^{\textrm{T}}$ and $[150,-150,0]^{\textrm{T}}$, respectively. The vertical positions of the explorers are all zero. The initial 3-D position of the relay agent is $[-25,-150,5]^{\textrm{T}}$. The consensus state $x_g$ is set as $[0, 0,0]^{\textrm{T}}$. The random disturbance $d_{i}(t)$ is a vector whose elements
are drawn at each time step $t$ from a standard uniform distribution
centered about the origin spanning $[-0.5\bar{d}_i, 0.5\bar{d}_i]$ for all $i\in F$, where $\bar{d}_1=0.04, \bar{d}_2=0.03$, and $\bar{d}_3=0.02$.

For approximate consensus, we consider the controller of the explorers in \eqref{explorer Controller}, where $P$ is as follows (computed from (\ref{ARE}) with $k=0.1$).\\
% put P here
\vspace{-0.05in}
\begin{equation}
P=\begin{bmatrix} 
0.23 & 0 & 0.22 & 0\\
0 & 0.23  & 0 & 0.22 \\
0.22 & 0 & 0.52 & 0 \\
0 & 0.22 & 0 & 0.52
\end{bmatrix}.
\end{equation}

We consider two different scenarios as follows.
    \begin{table}[]
		\centering
		\caption{MTL specifications $\phi_{\textrm{p}}$ and results in Scenario I.}
		\begin{tabular}{ll>{\raggedright\arraybackslash}p{20mm}}
		\toprule[2pt]    
			 ~~~~~MTL specification $\phi_{\textrm{p}}$ & \tabincell{c}{ cumulative\\control effort} \\ \hline
		    \tabincell{c}{$\phi_{\textrm{p},\textrm{I}}^1=\Box\Diamond_{[0, 20]}\big((y_0\in G_1)\vee (y_0\in G_2)\big)$\\ $\wedge\Box (y_0\in D)$}   & 70,943.32   \\ 
		    \tabincell{c}{$\phi_{\textrm{p},\textrm{I}}^2=\Box\Diamond_{[0, 10]}\big((y_0\in G_1)\vee (y_0\in G_2)\big)$\\ $\wedge\Box (y_0\in D)$}  & 101,079.03 \\ 
		     \tabincell{c}{$\phi_{\textrm{p},\textrm{I}}^3=\Box\Diamond_{[0, 6]}\big((y_0\in G_1)\vee (y_0\in G_2)\big)$\\ $\wedge\Box (y_0\in D)$} & 165,224.55 \\ \bottomrule[2pt]       
		\end{tabular}
		\vspace{-3mm}
		\label{result_TL1}  
	\end{table}
	
    \begin{table}[]
		\centering
		\caption{MTL specifications $\phi_{\textrm{p}}$ and results in Scenario II.}
		\begin{tabular}{ll>{\raggedright\arraybackslash}p{20mm}}
		\toprule[2pt]    
			 ~~~~~MTL specification $\phi_{\textrm{p}}$ & \tabincell{c}{cumulative\\control effort} \\ \hline
		    \tabincell{c}{$\phi_{\textrm{p},\textrm{II}}^1=\Box\Diamond_{[0, 6]}\big((y_0\in G_1)\vee (y_0\in G_2)\big)$\\ $\wedge\Box (y_0\in D)\wedge\lnot\Diamond\Box_{[0,1]} (y_0\in E)$}   & 644,670.20   \\ 
		    \tabincell{c}{$\phi_{\textrm{p},\textrm{II}}^2=\Box\Diamond_{[0, 6]}\big((y_0\in G_1)\vee (y_0\in G_2)\big)$\\ $\wedge\Box (y_0\in D)\wedge\lnot\Diamond\Box_{[0,2]} (y_0\in E)$}  & 165,984.72 \\ 
		     \tabincell{c}{$\phi_{\textrm{p},\textrm{II}}^3=\Box\Diamond_{[0, 6]}\big((y_0\in G_1)\vee (y_0\in G_2)\big)$\\ $\wedge\Box (y_0\in D)\wedge\lnot\Diamond\Box_{[0,3]} (y_0\in E)$} & 165,241.50 \\ \bottomrule[2pt]       
		\end{tabular}
		\vspace{-5mm}
		\label{result_TL2}  
	\end{table}

\noindent\textbf{Scenario 1}:
We consider three MTL specifications for $\phi_{\textrm{p}}$ as shown in Table \ref{result_TL1}. The relay agent needs to reach the charging station $G_1$ or $G_2$ at least once in any $6T_s$ time, and it should always remain in region $D$, where the two charging stations $G_1$ and $G_2$ are rectangular cuboids with length, width and height being 10, 10 and 5, centered at $[-100,50, 2.5]^{\textrm{T}}$ and $[125,0, 2.5]^{\textrm{T}}$, respectively. The region $D$ is a rectangular cuboid centered at $[0,0,7]^{\textrm{T}}$ with length, width and height being 300, 300 and 6, respectively (see Fig. \ref{fig_intro}). 

% This specification is expressed as
% \begin{align}\nonumber
% &\phi^1_{\textrm{p}}=\Box\Diamond_{[0, 6]}\big((y_0\in G_1)\vee (y_0\in G_2)\big)\wedge\Box (y_0\in D).
% \end{align} 

\noindent\textbf{Scenario 2}:
We consider three MTL specifications for $\phi_{\textrm{p}}$ as shown in Table \ref{result_TL2}. The relay agent needs to reach the charging station $G_1$ or $G_2$ at least once in any $6T_s$ time, always remain in region $D$ (same as in Scenario 1), and never stay in region $E$ for over $2T_s$ time, where the region $E$ is a rectangular cuboid centered at $[0,0,6]^{\textrm{T}}$ with length, width and height being 75, 75 and 4, respectively. 

% This specification is expressed as
% \begin{align}\nonumber
% \phi^2_{\textrm{p}}=&\Box\Diamond_{[0, 6]}\big((y_0\in G_1)\vee (y_0\in G_2)\big)\wedge\Box (y_0\in D)\\
% &\wedge\lnot\Diamond\Box_{[0,2]} (y_0\in E).
% \end{align} 

We set $R_f=R=5$, $V_T=1$,  $\eta=4$, $T_s=0.5$ and $N=20$. Fig. \ref{plot1} shows the simulation results in Scenario 1. We observe that the obtained control inputs of the relay agent gradually decrease as the explorers approach the goal region, $\Vert e_{1,i}(t)\Vert$ is uniformly bounded by $V_T=1$, and $\Vert e_{2,i}(t)\Vert$ gradually decreases and then oscillates when the explorers approach approximate consensus to the goal region. We measure the \textit{cumulative control effort} as $\sum_{j=0}^{\bar{N}}\norm{u_0^j}$, where $\bar{N}$ denotes the minimal time index such that $\norm{Cx_g-y_i(t[\bar{N}])}\le R_f$ for all $i\in F$. The results as shown in Table \ref{result_TL1} also show that the cumulative control effort for satisfying $\phi_{\textrm{p},\textrm{I}}^3$ is more than that for satisfying $\phi_{\textrm{p},\textrm{I}}^2$, which is still more than that for satisfying $\phi_{\textrm{p},\textrm{I}}^1$. This is consistent with the fact that $\phi_{\textrm{p},\textrm{I}}^2$ implies $\phi_{\textrm{p},\textrm{I}}^1$, and $\phi_{\textrm{p},\textrm{I}}^3$ implies both $\phi_{\textrm{p},\textrm{I}}^1$ and $\phi_{\textrm{p},\textrm{I}}^2$. 

Fig. \ref{plot2} shows the simulation results in Scenario 2. We observe that $\Vert e_{1,i}(t)\Vert$ is uniformly bounded by $V_T=1$ and $\Vert e_{2,i}(t)\Vert$ gradually decreases and then oscillates when the explorers approach approximate consensus to the goal region. The results as shown in Table \ref{result_TL2} also show that the cumulative control effort for satisfying $\phi_{\textrm{p},\textrm{II}}^1$ is more than that for satisfying $\phi_{\textrm{p},\textrm{II}}^2$, which is still more than that for satisfying $\phi_{\textrm{p},\textrm{II}}^3$. This is consistent with the fact that $\phi_{\textrm{p},\textrm{II}}^2$ implies $\phi_{\textrm{p},\textrm{II}}^3$, and $\phi_{\textrm{p},\textrm{II}}^1$ implies both $\phi_{\textrm{p},\textrm{II}}^2$ and $\phi_{\textrm{p},\textrm{II}}^3$. We also observe that more cumulative control effort is needed in Scenario 2 to satisfy the MTL specifications after the explorers arrive in region $E$ as the relay agent needs to get away from $E$ after each service to the explorers. Videos of the simulations in both scenarios are available in the CoppeliaSim environment\footnote{CoppeliaSim videos can be found at https://tinyurl.com/y32xgmtx.}.

% We also observe that the cumulative control effort for satisfying the MTL specifications in Scenario II are more than those for satisfying the MTL specifications in Scenario I as the MTL specifications in Scenario II imply the MTL specifications in Scenario I. 

\begin{figure}
	\centering
	\includegraphics[scale=0.35]{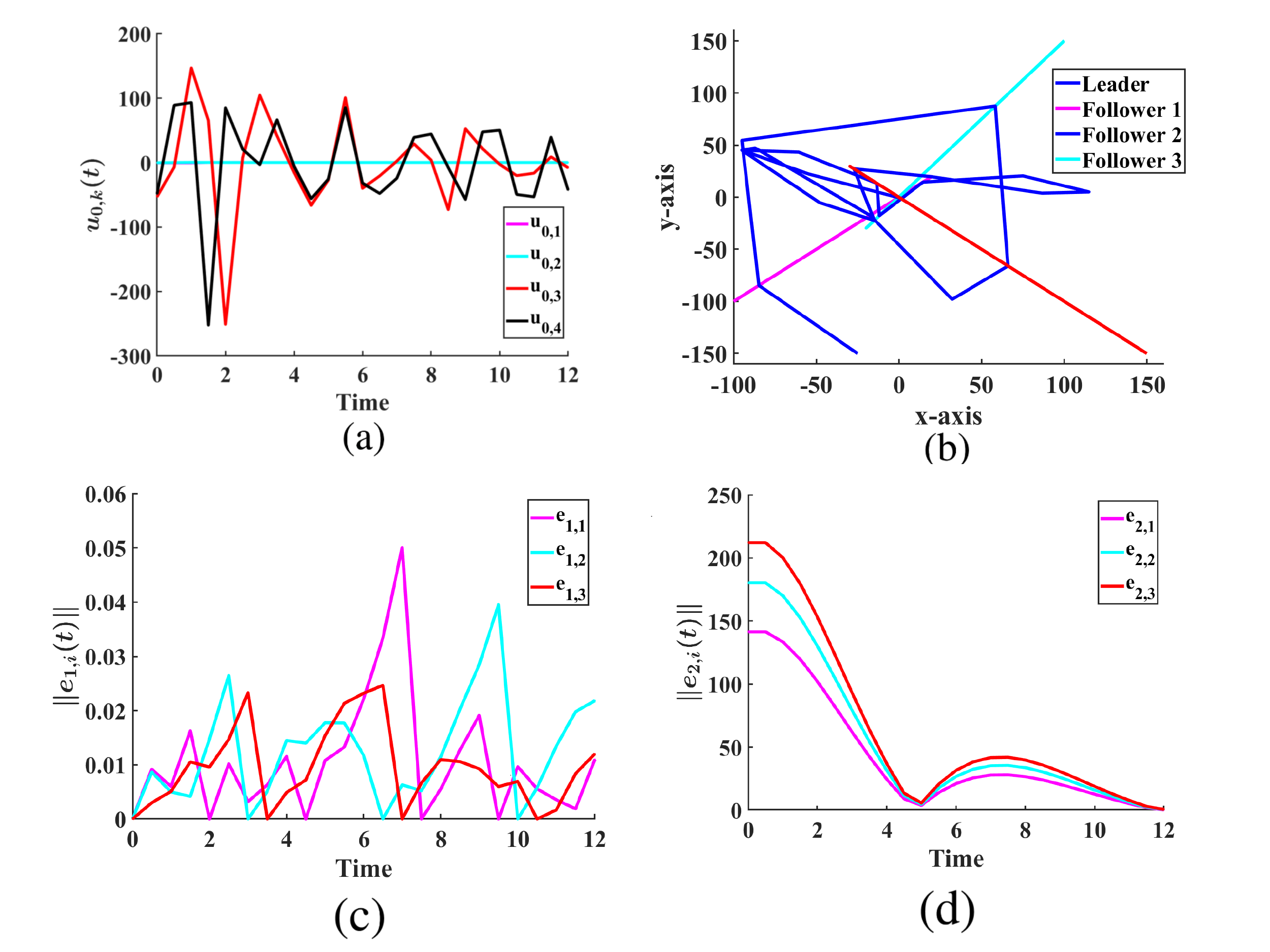}
	\caption{Results with MTL specification $\phi^1_{\textrm{p},\textrm{I}}$ for the practical constraints: (a) the obtained optimal inputs for the relay agent; (b)
	2-D planar plot of the trajectories of three explorers and a relay agent; (c) $\Vert e_{1,i}(t)\Vert$; (d) $\Vert e_{2,i}(t)\Vert$.}  
	\label{plot1}
	\vspace{-6mm}
\end{figure}

\begin{figure}
	\centering
	\includegraphics[scale=0.35]{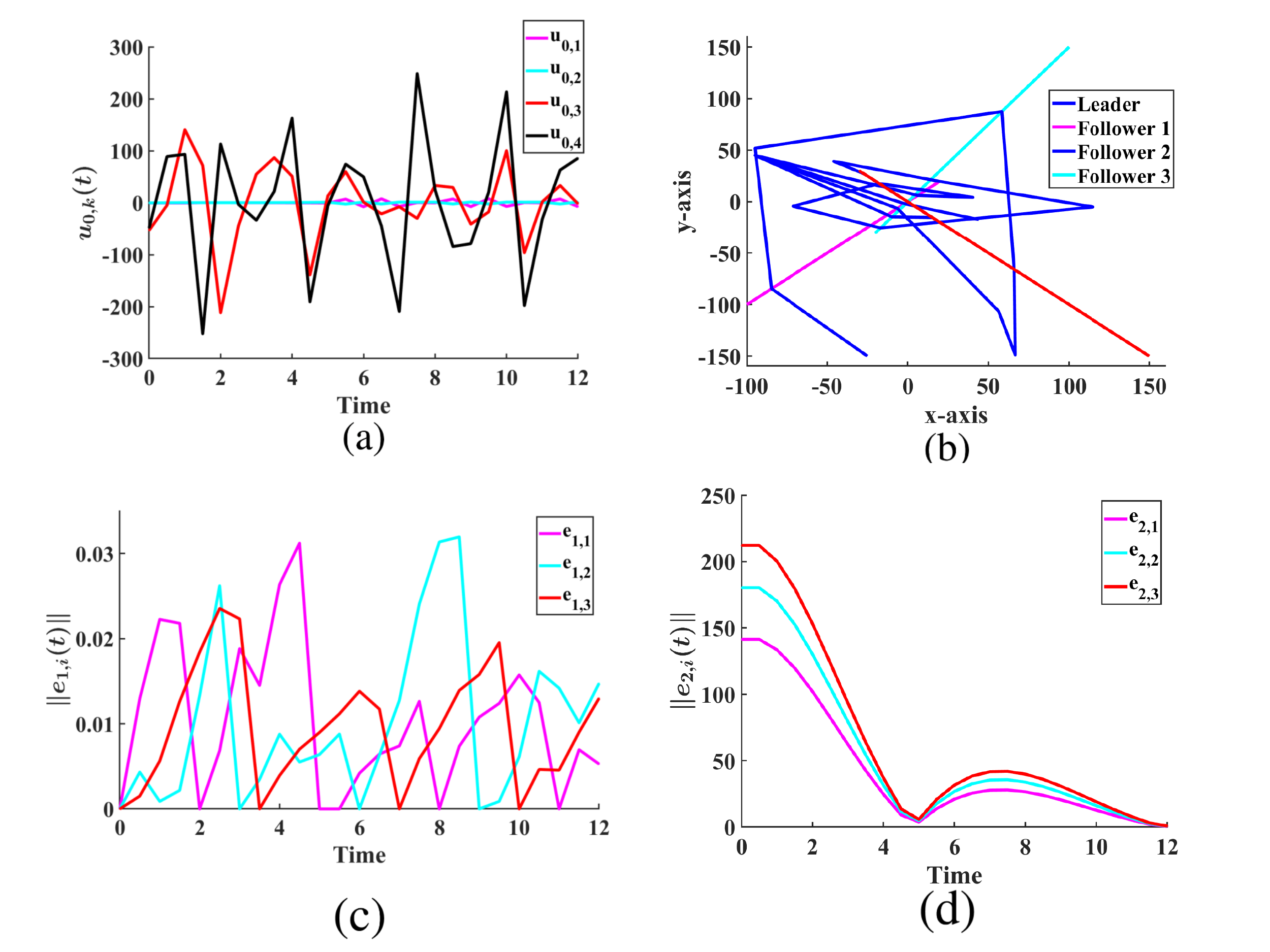}
	\caption{Results with MTL specification $\phi^1_{\textrm{p},\textrm{II}}$ for the practical constraints: (a) the obtained optimal inputs for the relay agent; (b)
	2-D planar plot of the trajectories of three explorers
	and a relay agent; (c) $\Vert e_{1,i}(t)\Vert$; (d) $\Vert e_{2,i}(t)\Vert$.}  
	\vspace{-7mm}
\label{plot2}
\end{figure}

\section{ Conclusion}
\label{conclusion}
We present a metric temporal logic approach for the controller synthesis of
a multi-agent system with
intermittent communication. We iteratively solve a sequence of mixed-interger linear programming problems for provably achieving correctness, stability of the switched system and approximate consensus of the explorers. Since the explorers are specified to reach
approximate consensus in this paper, we will investigate scenarios where controller synthesis
can be also conducted for the explorers with more complex
specifications.
% Future work will also extend the implementations to more realistic dynamic models for the explorers and experiments on the hardware testbed.

% \section{Acknowledgment}
% This research is supported in part by AFOSR award numbers FA9550-18-1-0109 and FA9550-19-1-0169, and NEEC award number N00174-18-1-0003. 

% Any opinions, findings and conclusions or recommendations expressed in this material are those of the author(s) and do not necessarily reflect the views of the sponsoring agency.

\bibliographystyle{IEEEtran}
\bibliography{zherefclean_submit}

\end{document}